\documentclass{article}

\usepackage{microtype}
\usepackage{graphicx}
\usepackage{subfigure}
\usepackage{booktabs} 

\usepackage{algorithm}
\usepackage{algorithmic}
\usepackage[utf8]{inputenc}
\usepackage{amsmath}
\usepackage{amsthm}
\usepackage{color}
\usepackage{mathrsfs}
\usepackage{todonotes,enumitem,rotating,multirow}
\newtheorem{theorem}{Theorem}
\newtheorem{proposition}{Proposition}
\newtheorem{lemma}{Lemma}
\newtheorem{definition}{Definition}
\newtheorem{hyp}{Hypothesis}

\DeclareMathOperator*{\argmax}{argmax}

\def\R{\mathbf{R}}

\def\calX{\mathcal{X}}
\def\calY{\mathcal{Y}}
\def\calZ{\mathcal{Z}}
\def\E{\mathbf{E}}

\def\OTc{\mathrm{OT}_c}

\usepackage{hyperref}

\usepackage{authblk}

\title{On the Existence of Optimal Transport Gradient for Learning Generative Models}

\author[1]{Antoine Houdard}
\author[1]{Arthur Leclaire}
\author[1]{Nicolas Papadakis}
\author[2]{Julien Rabin}
\affil[1]{Univ. Bordeaux, CNRS, IMB, UMR 5251, France}
\affil[2]{Normandie Univ., UniCaen, ENSICAEN, CNRS, GREYC, UMR 6072, France}
\date{}                     
\begin{document}

\maketitle

\begin{abstract}

The use of optimal transport cost for learning generative models has become popular with Wasserstein Generative Adversarial Networks (WGAN). Training of WGAN relies on a theoretical background: the calculation of the gradient of the optimal transport cost with respect to the generative model parameters. 
We first demonstrate that such gradient may not be defined, which can result in numerical instabilities during gradient-based optimization.
We address this issue by stating a valid differentiation theorem in the case of entropic regularized transport and specify conditions under which existence is ensured. 
By exploiting the discrete nature of empirical data, we formulate the gradient in a semi-discrete setting and propose an algorithm for the optimization of the generative model parameters.
Finally, we illustrate numerically the advantage of the proposed framework.

\end{abstract}

\section{Introduction}

Generative models are efficient tools to synthesize plausible samples that look similar to a given data distribution.  With the emergence of deep neural networks,  generative models such as Variational AutoEncoders~\cite{Kingma2014} or Generative Adversarial Networks (GAN)~\cite{goodfellow2014gan} now provide state-of-the-art results for most of machine learning methods dedicated to restoration and edition of signal, image and video data.

\paragraph{Wasserstein GAN.}
A popular instance of the GAN framework is given by 
Wasserstein GAN, or WGAN \cite{arjovsky2017wgan}. The  generative model of a WGAN is trained to provide a synthetic distribution that is close to a given data distribution with respect to an  optimal transport cost, e.g. the $1$-Wassertein distance as in \cite{arjovsky2017wgan}.

Several improvements and extensions of the original WGAN framework have been proposed in the literature. Most works have taken advantage of the particular  case of $1$-Wassersein distance.
Following Rubinstein-Kantorovitch duality, the $1$-Wassertein distance can be approximated using a $1$-Lipschitz discriminator network. The weight clipping strategy, originally  suggested in \cite{arjovsky2017wgan} to obtain a Lipschitz network, leads to convergence issues when training a WGAN.
Many technical improvements have therefore been proposed to both enforce the Lipschitz constraint of the discriminator network and stabilize the training~\cite{gulrajani2017improved,WGANLP, miyato2018spectral}. 

Other extensions come from the generalization to 
$p$-Wasserstein distances or regularized optimal transport costs. Among these extensions, the method of \cite{liu2018two} considers generic convex costs for optimal transport and relies on low dimentional discrete transport problems on batches during  the learning. In  \cite{korotin2019wasserstein}, the case of the 2-Wasserstein distance is tackled thanks to input convex neural networks \cite{amos2017input} and a cycle-consistency regularization. In order to have a differentiable distance, the use of entropic regularization of optimal transport has been proposed in different ways. The Sinkhorn algorihm \cite{cuturi2013sinkhorn} is plugged to compute regularized optimal transport between minibatches in \cite{genevay2018learning}, while a regularized WGAN loss function is considered in ~\cite{sanjabi2018convergence}. The  semi-discrete formulation of optimal transport has also been exploited to propose  generative models based on 2-Wasserstein distance~\cite{houdard2020wasserstein} or strictly convex costs \cite{Chen2019}.

\paragraph{Differentiation of optimal transport in WGAN training.} During the training of all the aforementioned WGAN methods, an optimal transport cost is minimized with respect to the generator parameters.
Using Fenchel's duality, the parameter estimation problem is reformulated as a min-max problem.  Parameter estimation is then performed with an alternating procedure that requires the gradient expression of the optimal transport cost.  
The computation of such a gradient involves the differentiation of a maximum and the use of an envelop theorem. 
{\em However, we argue that this envelop theorem may not stand in the general case.}
More importantly, failure cases can occur even under the theoretical assumptions made in~\cite{arjovsky2017wgan}.

In this work, we therefore propose an in-depth analysis of the gradient computation for optimal transport through the study of failure cases and their practical consequences. We  demonstrate that a stronger envelop theorem holds in the case of entropic regularization of optimal transport.
Finally, exploiting the discrete nature of the training data (corresponding to most practical cases), we single out the semi-discrete setting of optimal transport. 
In this setting, we derive an algorithm for the optimization of the generative model parameters.
Contrary to previous works as~\cite{sanjabi2018convergence}, we pay particular attention to the possible singularities that may prevent us from using the usual envelope theorem, and examine the impact of such irregularities in the learning process.

\paragraph{Problem statement.}

Considering the empirical measure $\nu$ associated with a discrete dataset $\{y_1,\ldots, y_n\}$ in a compact $\calY$, we aim at inferring a (continuous) generative model $g_\theta : \calZ \to \calX$ (defined from a latent space $\calZ$ to a compact~$\calX$ and depending on parameter $\theta$) whose output distribution $\mu_\theta$ best fits $\nu$.
The function~$g_{\theta}$ pushes a distribution $\zeta$ on the latent space $\calZ$ so that $\mu_{\theta}$ is the push-forward measure
$\mu_\theta = g_\theta \sharp \zeta$ (defined by $g_\theta\sharp\zeta(B) =\zeta(g_\theta^{-1}(B))$).
The goal is thus to compute a parameter $\theta$ that minimizes the optimal transport cost

\begin{equation}
\mathrm{OT}_c(\mu_\theta,\nu) = \inf_{\pi\in\Pi(\mu_\theta,\nu)} \int c(x,y)d\pi(x,y), \label{eq:OTcost}
\end{equation}
where $c :\calX\times\calY \to \R$ is a Lipschitz cost function and $\Pi(\mu_\theta,\nu)$ is the set of probability distributions on $\calX\times\calY$ having marginals $\mu_\theta$ and $\nu$. The direct minimization of $\mathrm{OT}_c(\mu_\theta,\nu)$ with respect to $\theta$ is a  difficult task. However, whenever $\calX$ and $\calY$ are compact, duality holds and the so-called semi-dual formulation of  optimal transport yields~\cite{santambrogio2015ot}
\begin{equation}\label{eq:semidualot}
\mathrm{OT}_c(\mu_\theta,\nu) = \max_{\psi \in L^\infty(\calY)} \int_\calX \psi^{c}(x) d \mu_\theta(x) + \int_\calY \psi(y) d \nu(y),
\end{equation}
where $\psi \in L^\infty(\calY)$ is called a {\em dual potential} and  
\begin{equation}\label{eq:ctransf}\psi^{c}(x) = \min_{y\in\calY} \left[ c(x,y) - \psi(y)\right]
\end{equation}
is the $c$-transform of $\psi$. 
Any dual potential satisfying the max in \eqref{eq:semidualot} will be referred to as a {\em Kantorovitch potential}. 
Denoting as $$F(\psi, \theta)=\int_\calX \psi^{c}(x) d \mu_\theta(x) + \int_\calY \psi(y) d \nu(y),$$ the differentiation with respect to $\theta$ of the optimal transport cost $\OTc(\mu_\theta,\nu)$ can be related to the differentiation under the maximum of the quantity
\begin{equation}
W_c(\theta) = \max_{\psi \in L^\infty(\calY)} F(\psi, \theta).
\end{equation}
We now introduce the following  result from \cite{arjovsky2017wgan} that will be used all along the paper.
\begin{theorem}[Envelop theorem]\label{thm:env}
Let $\theta_0$ and $\psi^*_0$ verifying $W_c(\theta_0) = F(\psi^*_0,\theta_0)$. 
If $W_c$ and ${\theta \mapsto F(\psi^*_0,\theta)}$  are both differentiable at $\theta_0$, then
\begin{equation}\label{eq:gradOT}
\nabla W_c(\theta_0) = \nabla_\theta F(\psi^*_0,\theta_0).
\end{equation}
\end{theorem}
The proof of this \emph{weak} version of the envelop theorem is straightforward.
\begin{proof} 
Let $\theta_0$ and $\psi_0$ be as in the hypothesis of the theorem. Let us define 
\begin{equation}
H : \theta \mapsto F(\psi^*_0,\theta) - W(\theta).
\end{equation}
For all $\theta$, $H(\theta) \leq 0$  from the definition of $W$ and $H(\theta_0) = 0$ from definition of $\psi_0$. Since we assume $H$ differentiable at $\theta_0$, we get $\nabla H(\theta_0) = 0$ and the result follows.
\end{proof}
However, there may exist no couple $(\theta_0, \psi^*_0)$ for which~\eqref{eq:gradOT} holds, even in cases that would seem favorable. 
This scenario can indeed occur for 
$W_c$ being differentiable everywhere, or for $F$ admitting partial derivative in $\theta$ for almost every $\theta$ and $\psi$.

\paragraph{Outline.} 
The main objective of this paper is to identify what may prevent the existence of the right-hand side quantity $\nabla_\theta F(\psi_0,\theta_0)$ in \eqref{eq:gradOT}. 
In Section~\ref{sec:failex}, we study a simple counterexample where $F(\psi^*_0,\cdot)$ is never differentiable at $\theta_0$ and we discuss the impact of this property on the actual Wasserstein GAN procedure. We then demonstrate in Section~\ref{sec:reg} that the entropic regularization of optimal transport allows the application of a stronger version of the envelop theorem and ensures the differentiability of the regularized optimal transport cost. The formulation of the gradient is nevertheless difficult to exploit in practice, as it requires to estimate an expectation over the whole target distribution~$\nu$. We therefore propose in Section~\ref{sec:practicaluse} to take advantage of the discrete nature of the target data $\{y_1,\ldots,y_n\}$ to derive a feasible algorithm in the semi-discrete setting of optimal transport. We finally illustrate the benefits of the proposed framework through numerical experiments in Section~\ref{sec:experiments}.

\section{Case study on a synthetic example\label{sec:failex}}

In this section, we present an example that brings down 
the theoretical assumption made in~\cite{arjovsky2017wgan}.
In Section~\ref{sec:synthex} we demonstrate that the envelop Theorem~\ref{thm:env} does not hold for this example. More precisely, we show that even if the gradient of the optimal transport cost with respect to the parameter exists for any $\theta$, the gradient of the function $F$ does not exist for any $\theta$. 
In Section~\ref{sec:convissue}, we emphasize that for this example, the generative model satisfies the hypothesis made in~\cite{arjovsky2017wgan} and show that it can lead to  convergence instabilities during the training. 

\begin{figure}[t]
    \centering
    \begin{tabular}{cc}
    \includegraphics[width=0.48\linewidth]{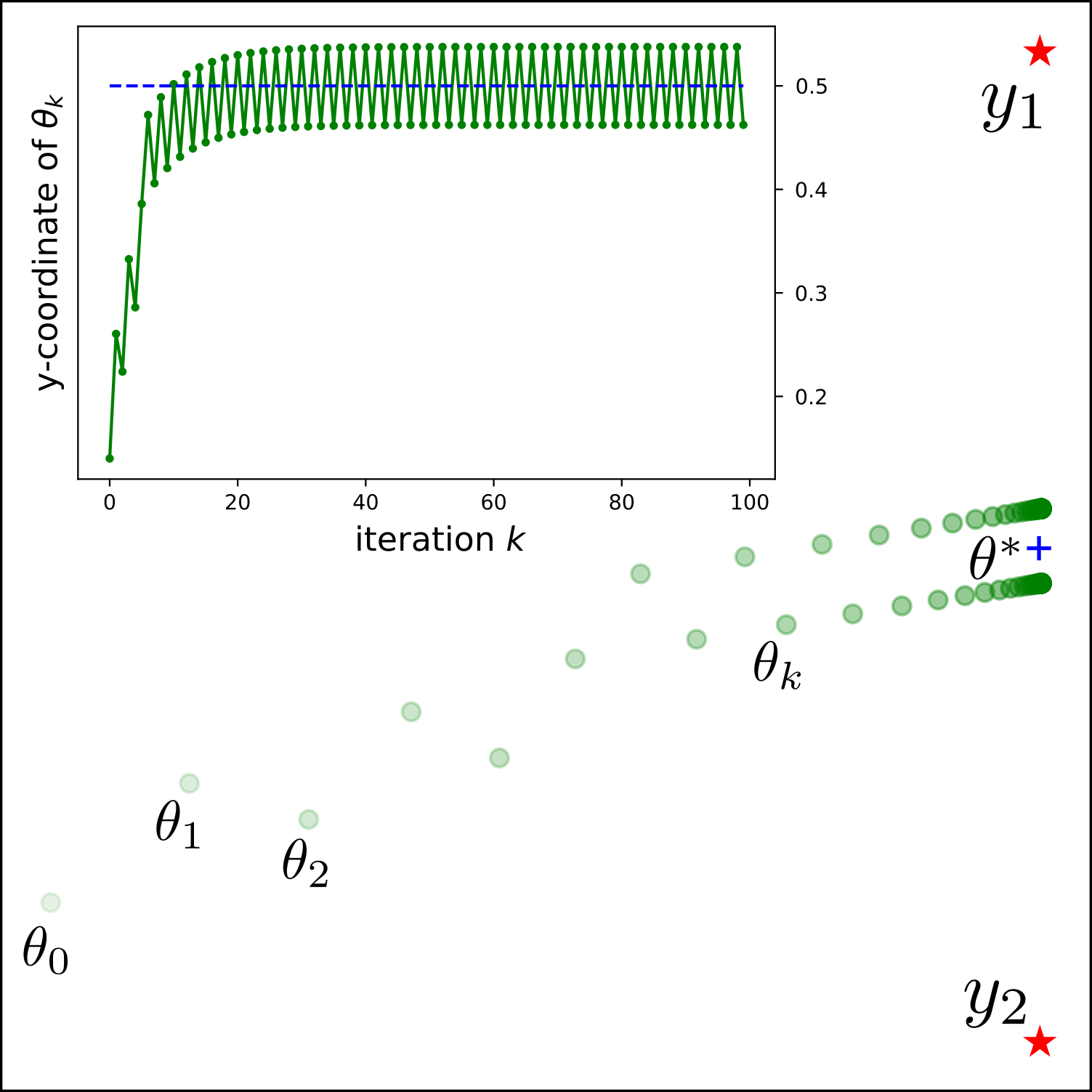}
    &\includegraphics[width=0.48\linewidth]{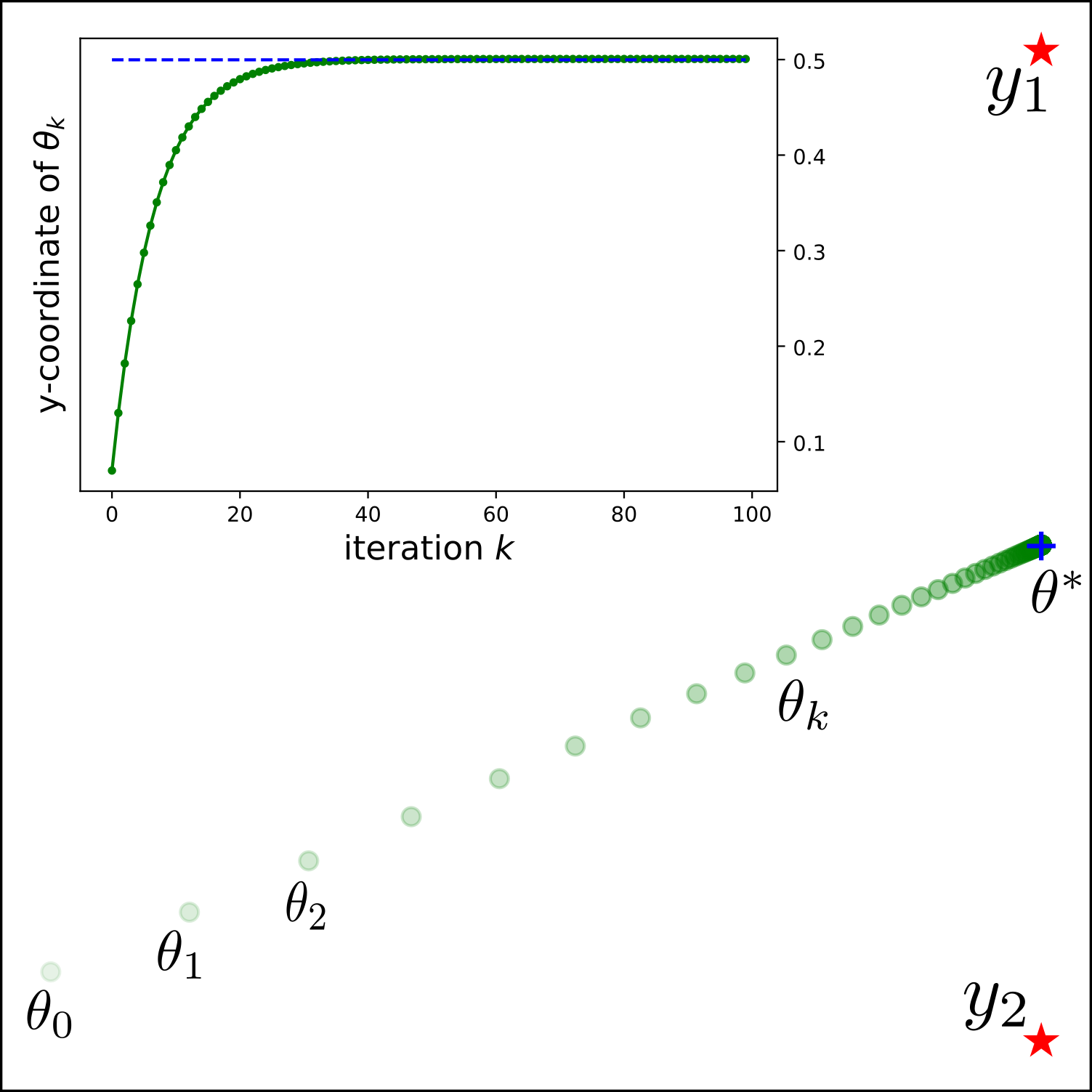}%
    \\
    $\text{OT}_{c}$
    &
    $\text{OT}_{c}^\lambda$
    \end{tabular}
    \label{fig:synthex}
    \caption{Plot of the trajectory of the parameter $\theta^k$ during optimization of the generative model $g_\theta(z) = z - \theta$ for two training points $\{y_1,y_2\}$. Left: as predicted by Proposition~\ref{prop:counterexample}, the process does not converge for the optimal transport $\text{OT}_c$ with quadratic cost $c = \|.\|^2$. Right: as supported by Theorem~\ref{thm:gc1}, the training converges to the solution $\theta^*=(0,0.5)$  when considering regularized optimal transport $\text{OT}_c^\lambda$. See Section~\ref{sec:synth_solved} for more details. 
    }
    \end{figure}

\subsection{An example with discrete measures}\label{sec:synthex}

Let us consider a simple optimal transport problem between a Dirac $\delta_\theta$ located at $\theta\in\mathbf{R}^2$ and a sum of two Diracs at positions $y_1 \neq y_2 \in \mathbf{R}^2$ (see Figure~\ref{fig:synthex}). This setting corresponds to a latent code $\zeta = \delta_0$ and a generator $g_\theta$ defined by $g_\theta(z) = z - \theta$ for all $z\in \calZ$. We show that in this case the hypotheses of Theorem~\ref{thm:env} are satisfied for no~$\theta_0$. More precisely we show the following result.\\
\begin{proposition}\label{prop:counterexample}
Let $\mu_\theta = g_\theta\sharp\zeta = \delta_\theta$ and $\nu = \frac{1}{2}\delta_{y_1} + \frac{1}{2}\delta_{y_2}$ and consider a cost $c(x,y) = \|x-y\|^p$, $p\geq 1$ then\begin{itemize}
\item $\theta \mapsto W_c(\theta) = \OTc(\mu_{\theta},\nu)$ is differentiable at any $\theta \notin \{y_1,y_2\}$ for $p=1$, and at any $\theta$ for $p>1$,
\item for any $\theta_0$ and any $\psi_0^* \in \argmax_\psi F(\psi,\theta_0)$, the function $\theta \mapsto F(\psi_0^*, \theta)$ is \textbf{not} differentiable at $\theta_0$.
\end{itemize}
Hence relation \eqref{eq:gradOT} never stands.
\end{proposition}

\begin{proof}
The dual formulation of optimal transport writes \begin{equation}
OT_c(\mu_\theta, \nu) = \max_{\psi\in\R^2} F(\psi,\theta)
\end{equation} with $F(\psi,\theta) = \min_{i=1,2} \left[c(\theta,y_i) - \psi_i\right] + \frac{\psi_1 + \psi_2}{2}$. We can therefore write 
\begin{equation*}
\begin{split}
&F(\psi,\theta)=\\ &
  \begin{cases}
    c(\theta,y_1) + \frac{\psi_2 - \psi_1}{2} & \text{if } c(\theta,y_1) - \psi_1 \leq c(\theta,y_2) - \psi_2 \\
    c(\theta,y_2) + \frac{\psi_1 - \psi_2}{2} & \text{if } c(\theta,y_2) - \psi_2 \leq c(\theta,y_1) - \psi_1. \\
  \end{cases}
  \end{split}
\end{equation*}
Which yields $2 F(\theta, \psi) \leq c(\theta,y_1) + c(\theta,y_2)$, where equality is reached for $\psi_1 = c(\theta,y_1)$ and $\psi_2 = c(\theta, y_2)$. As a consequence, the optimal transport cost writes
\begin{equation}\label{eq:sol_unreg}
\OTc(\mu_\theta,\nu) = \frac{c(\theta,y_1) + c(\theta,y_2)}{2},
\end{equation}
and it is  differentiable at any $\theta \notin \{y_1,y_2\}$, and even at $\theta \in \{y_1, y_2\}$ for $p>1$. This shows the first point of the proposition.
 
Let us now fix $\theta_0$. One can show that
\begin{equation*}
\argmax_\psi F(\psi,\theta_0) = \left\{ \left( c(\theta_0, y_1), c(\theta_0,y_2) \right) + C, C\in\R\right\}.
\end{equation*}
Next we fix $\psi^* \in \argmax_\psi F(\psi,\theta_0)$, so that $\psi^*_2 - \psi^*_1 = c(\theta_0,y_2) - c(\theta_0, y_1)$. We define the Laguerre cells for $\psi^*$ 
\begin{equation*}
L_i(\psi^*) = \{x  ~|~ \forall k\neq i, ~ c(x, y_i) - \psi^*_i < c(x,y_k) - \psi^*_k\}
\end{equation*}
as well as the boundary set
\begin{equation*}
H(\psi^*) = \{x  ~|~ c(x, y_1) - \psi^*_1 = c(x,y_2) - \psi^*_2\}.
\end{equation*}
The function $F$ thus writes
\begin{equation*}
\begin{split}
&F(\psi,x)=\\ &
  \begin{cases}
    c(x,y_1) + \frac{\psi_2^* - \psi_1^*}{2} & \text{if } x\in L_1(\psi^*) \\
    c(x,y_2) + \frac{\psi_1^* - \psi_2^*}{2} & \text{if } x\in L_2(\psi^*)\\
    c(x,y_1) + \frac{\psi_2^* - \psi_1^*}{2} = c(x,y_2) + \frac{\psi_1^* - \psi_2^*}{2} & \text{if } x\in H(\psi^*)
  \end{cases}
  \end{split}
\end{equation*}

Notice that both Laguerre cells 
$L_1(\psi^*)$ and $L_2(\psi^*)$ are open sets. 
The map $x \mapsto F(\psi^*,x)$\ is differentiable on these cells and we have
\begin{equation}
\forall x \in L_i(\psi^*), ~ \nabla_x F(\psi^*,x) = \nabla_x c(x,y_i).
\end{equation}

Therefore, in the neighborhood of a point $x \in H(\psi^*)$, the restrictions of $x \mapsto F(\psi^*,x)$ on $L_1(\psi^*)$ and $L_2(\psi^*)$ are smooth but their gradients do not agree at the boundary in-between.
As a consequence, $\theta \mapsto F(\psi^*,\theta)$ is not differentiable at any $\theta_0 \in H(\psi^*)$ which shows the second point of the proposition.
\end{proof}

\subsection{Consequences on the convergence of WGAN\label{sec:convissue}}

Wasserstein GAN approaches exploit the gradient $\nabla_\theta F$ in \eqref{eq:gradOT}  to train a generative model $g_\theta$. Relation~\eqref{eq:gradOT} is stated in \cite{arjovsky2017wgan} under the following hypothesis relying on the Lipschitz properties of the  generator. 

\begin{hyp}\label{hyp:lip}
$g:\Theta\times \calZ \to \calX$ satisfies Hypothesis~\ref{hyp:lip} if there exists $L:\Theta\times \calZ \to \R_+$ such that $\forall \theta$, there exists a neighborhood $\Omega$ of $\theta$ such that $\forall \theta' \in \Omega$ and $\forall z,z'\in \calZ$
\begin{equation}
\|g(\theta, z) - g(\theta',z')\| \leq L(\theta,z)\| (\theta,z)-(\theta',z')\|
\end{equation}
\begin{equation}
\text{and for all $\theta$} \qquad \mathbf{E}_{Z\sim\zeta}\left[L(\theta,Z)\right] := L(\theta) < \infty.
\end{equation}
\end{hyp}

In our previous example of Section \ref{sec:synthex}, we demonstrate  that 
the envelop theorem does not hold for the   generator $g_\theta(z)= z-\theta$. This generator nevertheless satisfies  Hypothesis \ref{hyp:lip} as it is locally Lipschitz with constant  $L(\theta,z)=1$.
Hence, {\bf the example we designed raises a flaw of Theorem 3 in} \cite{arjovsky2017wgan}, although this result is used to justify the existence of gradients for the training of most of WGAN-based methods.

Let us now emphasize an important point.
The envelop theorem used in \cite{arjovsky2017wgan} requires both terms of \eqref{eq:gradOT} to exist simultaneously. In fact, the differentiability of $F(\psi,\cdot)$ for a given $\psi$ is only guaranteed at almost every~$\theta$. Therefore, if we consider the optimal potential $\psi_0^*$ corresponding to a given $\theta_0$ the function $F(\psi_0^*,\cdot)$ may not be differentiable at $\theta_0$. This is exactly what happens in our synthetic example. 

Such issue occurs in any case where the generated distribution is not absolutely continuous with respect to the Lebesgue measure. This is typically the case when the latent space $\calZ$ of a generative model is lower dimensional than the ambient data space $\calX$.

\section{Gradient of the regularized optimal transport cost}
\label{sec:reg}

In this section, we propose to exploit the differential properties of the entropic regularization of optimal transport, in order to have a well-posed training procedure for WGAN. More precisely, we show that under some assumptions on the generator $g_\theta$, we can compute the gradient of the regularized optimal transport cost between the generated measure $\mu_\theta$ and the target measure $\nu$. 
In Section~\ref{sec:regOT}, we first recall results on the entropic regularization of optimal transport. In particular, we formulate  the regularized optimal transport cost as the maximum  of a function $F_c^\lambda(\psi,\theta)$ over $\psi$. 
Then, we show in Section~\ref{sec:diffF} the differentiability of $F_c^\lambda$ with respect to $\theta$, under the Hypothesis~\ref{hyp:lip} on the generator~$g_\theta$. In Section~\ref{sec:diffOT} we demonstrate the differentiability of the optimal transport cost under different hypotheses, relating the gradient of the optimal transport cost to the gradient of $F_c^\lambda$. 
In Section~\ref{sec:synth_solved} we finally apply these results to the synthetic example from Section~\ref{sec:synthex}.

\subsection{Definitions and requirements}\label{sec:regOT}

We consider from now on the entropic regularization of optimal transport. 
\begin{definition}[OT cost with entropic regularization \cite{genevay2019thesis}]
For $\lambda >0$, the regularized OT cost is defined by
\begin{equation}
\mathrm{OT}_c^\lambda(\mu,\nu) = \inf_{\pi\in\Pi(\mu,\nu)} \int c(x,y)d\pi(x,y) + \lambda H(\pi|\mu\otimes\nu)  \label{eq:ot_reg}
\end{equation}
where 
\begin{equation} \label{eq:relative_entropy}
 H(\pi|\mu\otimes\nu) = \int \left(\log\left(\frac{ d\pi(x,y)}{d(\mu(x)\nu(y))} \right) - 1 \right)d\pi(x,y) + 1
\end{equation}
is the relative entropy of the transport plan $\pi$ w.r.t. the product measure $\mu\otimes\nu$.
\end{definition}
Note that $\lambda>0$  makes the problem~\eqref{eq:ot_reg} strictly convex, whereas $\lambda=0$ corresponds to the non-regularized case from \eqref{eq:OTcost}.  In the following, we may either refer to $OT_c$ or $OT_c^0$. We consider the  relative entropy as in \cite{genevay2019thesis} which differs from the regularized formulation originally introduced in \cite{cuturi2013sinkhorn}. As shown in \cite{genevay2019thesis}, the  dual formulation of the problem with relative entropy can be expressed as the maximum of an expectation. In order to write the semi-dual  formulation as in \eqref{eq:semidualot}, we introduce the regularized $c,\lambda$-transform:
\begin{definition}[regularized $c,\lambda$-transforms]
Let $\psi:\mathcal{Y} \to \mathbf{R}$ we define the $c,\lambda$-transform for $\lambda >0$ by
\begin{equation}\label{eq:def_psicl}
    \psi^{c,\lambda}(x) = -\lambda\log\left(\int_\calY \exp\left(\frac{\psi(y) - c(x,y)}{\lambda}\right)d \nu(y)\right).
\end{equation}
\end{definition}
This formula with $\lambda=0$ yields the $c$-transform $\psi^c$ introduced earlier. The semi-dual formulation then writes.
\begin{proposition}[Semi-dual formulation of regularized transport] \label{th:reg_ot_semidual}
The primal problem \eqref{eq:ot_reg} is equivalent to the semi-dual problem
\begin{equation}\label{eq:semidual_l}
\mathrm{OT}_c^\lambda(\mu,\nu) = \max_{\psi\in L^\infty(\mathcal{Y})} \int_\calX \psi^{c,\lambda}(x)d\mu(x) + \int_\calY\psi(y) d\nu(y).
  \end{equation}
\end{proposition}

We then recall the following theorem that will be used in the next section.
\begin{theorem}[Existence and uniqueness of the dual solution \cite{genevay2019thesis}]
Providing $c\in L^\infty(\calX\times\calY)$, the dual problem \eqref{eq:semidual_l} admits a solution $\psi^*\in L^\infty(\nu)$ which is unique $\nu-a.e.$ up to an additive constant. Any solution $\psi^*$ will be referred to as a Kantorovitch potential.
\end{theorem}

As in the un-regularized case, we consider $\mu_\theta = g_\theta\sharp\zeta$ defined through a generative model. Defining
\begin{equation}\label{eq:def_Fcl}F_c^\lambda(\theta, \psi)=\E_{Z\sim \zeta} [\psi^{c,\lambda}(g_\theta(Z))] + \E_{Y\sim \nu} [\psi(Y)],\end{equation}
our main goal is to differentiate with respect to the parameter $\theta$ the quantity
\begin{align}\label{eq:def_Wl}
W_c^\lambda(\theta) &= \mathrm{OT}_c^\lambda(\mu_\theta,\nu) =  \max_{\psi\in L^\infty(\calY)} F_c^\lambda(\theta, \psi).
\end{align}

Before going further let us state the following Lemma about regularity of the Kantorovitch potentials 

\begin{lemma}\label{lem:modulus} 
If $c$ is uniformly continuous on $\calX\times\calY$, there exists a function $\omega : \mathbf{R}_+ \to \mathbf{R}_+$ not-decreasing, continuous at $0$ with $\omega(0)=0$  called \emph{modulus of continuity} of $c$ such that $\forall x,x' \in\calY$ and $\forall y,y'\in\calY$,
$$|c(x,y) - c(x',y')| \leq \omega(\|x-x'\| + \|y-y'\|).$$
In this case, any Kantorovicth potential shares this same modulus of continuity.
\end{lemma}

This lemma contains two different points to demonstrate. The first one can be shown using a standard analysis result stating that any uniformly continuous function admits a modulus of continuity. The second one is specific to optimal transport theory. Let us demonstrate the two propositions separately.

\begin{proposition}
Let $(\mathcal{M},|\cdot|)$ be a metric space and $f: \mathcal{M} \to\mathbf{R}$ be a uniformly continuous function. Then $f$ admits a modulus of continuity $\omega : \mathbf{R}_+ \to \mathbf{R}_+$ conitnuous at $0$ with $\omega(0)=0$.
\end{proposition}

\begin{proof}
From uniform continuity, let $\delta > 0$ be such that $\|a - b\| < \delta$ implies $|f(a) - f(b)|< 1$. Let us take $\omega$ defined for all $t \leq \delta$ by $\omega(t) = \sup\left\{ |f(a) - f(b)| ~|~ \|a - b\| \leq t \right\}$ and $\omega(t) = 2\sup_{x\in\mathcal{M}}\|f(x)\|$ whenever $t>\delta$.
From this definition, one obtains  $\omega$ positive, non-decreasing and the fact that $\forall a,b \in \mathcal{M}$, $|f(a) - f(b)| \leq \omega( \|a - b\| )$ and $\omega(0) =0$. Let us now show that $\omega$ is continuous at $0$, that is to say that $\omega(t)\to 0$ when $t\to0$. Since $\omega$ is positive and monotone, the limit exists and we denote $l = \lim_{t\to0}\omega(t)$. Let $\varepsilon >0$, the uniform continuity yields that there exists $\eta>0$ such that $\|a - b\| < \eta$ implies $|f(a) - f(b)|< \varepsilon$. Then for any $t < \min(\delta, \eta)$ we have $\omega(t) \leq \varepsilon$ and therefore $l \leq\varepsilon$. This shows that $l=0$ and concludes the proof.
\end{proof}

\begin{proposition}
Let $\omega$ be a modulus of continuity of the cost $c$, then any Kantorovitch potential of $\OTc^\lambda$ shares the same modulus of continuity as $c$.
\end{proposition}

\begin{proof}
Let $\psi$ be a solution of the semi-dual problem of entropic regularized optimal transport (Eq.~(13) in the main paper). There exists $\varphi \in L^\infty(\calX)$ such that $\psi = \varphi^{c,\lambda}$. Then from the definition of $\omega$ in Lemma 1, we have for all $x\in\calX$ and $y,y'\in\calY$:
$$ \varphi(x) - c(x,y) = \varphi(x) - c(x,y) + c(x,y') - c(x,y')  \leq \omega(\|y-y'\|) + \varphi(x) - c(x,y').$$
This leads to
$$\exp\left(\frac{ \varphi(x) - c(x,y)}{\lambda} \right) \leq  \exp\left(\frac{\omega(\|y-y'\|)  + \varphi(x) - c(x,y')}{\lambda}\right).$$
Taking the  expectation for $x\sim\mu$ yields
\begin{equation*}
\resizebox{\textwidth}{!}{$
\E_{X\sim\mu}\left[\exp\left(\frac{\varphi(X) - c(X,y)}{\lambda} \right)\right]\leq  \exp\left(\frac{\omega(\|y-y'\|)}{\lambda}\right) \E_{X\sim\mu}\left[\exp\left(\frac{\varphi(X) - c(X,y')}{\lambda}\right)\right].$
}
\end{equation*}
Finally applying $-\lambda\log$, we get
$$ \varphi^{c,\lambda}(y') - \omega(\|y-y'\| \leq \varphi^{c,\lambda}(y).$$
Switching the roles of $\varphi^{c,\lambda}(y)$ and $\varphi^{c,\lambda}(y')$, we finally obtain the desired result:
$$ |\varphi^{c,\lambda}(y') -  \varphi^{c,\lambda}(y)| \leq \omega(\|y-y'\|).$$
\end{proof}

Finally, the following theorem permits to control the variations of $W_c^\lambda$ with the variations of $\theta \to g_\theta$

\begin{theorem}\label{thm:varW}
If $c\in\mathscr{C}^1(\calX \times \calY)$ and if $\calX$ and $\calY$ are compact, there exists $\kappa$ such that $c$ is $\kappa$-Lipschitz and for all $\theta_1,\theta_2$
\begin{equation}
|W_c^\lambda(\theta_1) - W_c^\lambda(\theta_2)| \leq \kappa\E_{Z\sim\zeta}\left[\|g_{\theta_1}(Z)-g_{\theta_2}(Z)\|\right].\label{eq:varW}
\end{equation}
\end{theorem}

\begin{proof} As $c$ is a $\mathscr{C}^1$ function on a compact set, it is $\kappa$-Lipschitz with $\kappa=\sup_{\calX\times\calY} \|Dc(x,y)\|$, where $Dc(x,y)$ is the differential of $c$ at $(x,y)$.
Let us first prove that for all $\theta_1,\theta_2$ 
\begin{equation}
|W_c^\lambda(\theta_1) - W_c^\lambda(\theta_2)| \leq \kappa W_1(\mu_{\theta_1},\mu_{\theta_2}),
\end{equation}
where $W_1(\mu_{\theta_1},\mu_{\theta_2})$ is the 1-Wasserstein distance between $\mu_{\theta_1}$ and $\mu_{\theta_2}$ (i.e. the Wasserstein distance associated with the cost $(x,y) \mapsto \|x-y\|$). Taking $\psi_1$ a Kantorovitch potential for $W_c^\lambda(\theta_1)$ and $\psi_2$ a Kantorovitch potential for $W_c^\lambda(\theta_2)$, we can  write
\begin{align*}
W_c^\lambda(\theta_1) &- W_c^\lambda(\theta_2)\\
 &= \E_{\mu_{\theta_1}}\left[\psi_1^{c,\lambda}(X) \right]  + \E_{\nu}\left[\psi_1(Y) \right]  - \E_{\mu_{\theta_2}}\left[\psi_2^{c,\lambda}(X) \right] - \E_{\nu}\left[\psi_2(Y) \right]\\
&=  \E_{\mu_{\theta_1}}\left[\psi_1^{c,\lambda}(X) \right] - \E_{\mu_{\theta_2}}\left[\psi_1^{c,\lambda}(X) \right]\\
&+ \left( \E_{\mu_{\theta_2}}\left[\psi_1^{c,\lambda}(X) \right]  + \E_{\nu}\left[\psi_1(Y) \right]  - \E_{\mu_{\theta_2}}\left[\psi_2^{c,\lambda}(X) \right] - \E_{\nu}\left[\psi_2(Y) \right] \right).
\end{align*}
By optimality of $\psi_2$ for $W_c^\lambda(\theta_2)$, the sum of terms in the last parenthesis is non-positive. By switching $\theta_1$ and $\theta_2$ in the previous formula, we get
\begin{equation*}
\resizebox{\textwidth}{!}{$
\E_{\mu_{\theta_1}}\left[\psi_2^{c,\lambda}(X) \right] - \E_{\mu_{\theta_2}}\left[\psi_2^{c,\lambda}(X) \right]  \leq  W_c^\lambda(\theta_1) - W_c^\lambda(\theta_2) \leq \E_{\mu_{\theta_1}}\left[\psi_1^{c,\lambda}(X) \right] - \E_{\mu_{\theta_2}}\left[\psi_1^{c,\lambda}(X) \right]$}
\end{equation*}

and thus
\begin{align}
 | W_c^\lambda(\theta_1) - W_c^\lambda(\theta_2) | &\leq \sup_{\psi\in L^\infty(\calY)} \left| \E_{\mu_{\theta_1}}\left[\psi^{c,\lambda}(X) \right] - \E_{\mu_{\theta_2}}\left[\psi^{c,\lambda}(X) \right]\right| 
\end{align}
From Proposition~2 all the $\psi^{c,\lambda}$ shares the modulus of continuity of the cost $c$, which is $\kappa$-Lipschitz and noted $\psi^{c,\lambda} \in {Lip}_\kappa$. Hence we can restrict the supremum to $\kappa$-Lipschitz functions and we get that
\begin{align}
| W_c^\lambda(\theta_1) - W_c^\lambda(\theta_2) |  &\leq \sup_{\varphi\in Lip_\kappa} \left| \E_{\mu_{\theta_1}}\left[\varphi(X) \right] - \E_{\mu_{\theta_2}}\left[\varphi(X) \right]\right|\\
&\leq \kappa \sup_{\varphi\in Lip_1} \left| \E_{\mu_{\theta_1}}\left[\varphi(X) \right] - \E_{\mu_{\theta_2}}\left[\varphi(X) \right]\right|
= \kappa W_1(\mu_{\theta_1},\mu_{\theta_2}).
\end{align}

Then using the transport plan $\gamma = (g_{\theta_1}\sharp\zeta, g_{\theta_2}\sharp\zeta)$, which is admissible for $W_1(\mu_{\theta_1},\mu_{\theta_2})$, we get  

$$  W_1(\mu_{\theta_1},\mu_{\theta_2}) \leq \E_{Z\sim\zeta}\left[\|g_{\theta_1}(Z)-g_{\theta_2}(Z)\|\right] ,$$
which gives the result.
\end{proof}

Hence a Lipschitz regularity on $g$ implies the same regularity on $W_c^\lambda$.
In view of establishing the differentiability of $W_c^\lambda$, we first study the differentiability  of $F_c^\lambda(\theta,\psi)$  with respect to $\theta$ in the next section.

\subsection{Regularity of $F_c^\lambda$}\label{sec:diffF} 

Before coming to our object of interest, the gradient of $W_c^\lambda(\theta)$ defined in \eqref{eq:def_Wl} with respect to the generator parameters $\theta$, we first study the regularity of $F_c^\lambda$ defined by~\eqref{eq:def_psicl}.
Compared to the un-regularized case, the entropic regularization changes the minimum in the $c$-transform \eqref{eq:ctransf} into a soft-minimum in the $c,\lambda$-transform \eqref{eq:def_psicl}. With this additional regularity property, we show the following differentiability theorem on the functional $F_c^\lambda$.
\begin{theorem} \label{thm:gradF} Let $c\in \mathscr{C}^1(\calX\times\calY)$ for $\calX$ and $\calY$ compact and $g:\Theta\times\calZ$ satisfying Hypothesis~\ref{hyp:lip}. Then for almost every~$\theta_0$, for any $\psi\in L^\infty$ and $\lambda>0$ the function $\theta \mapsto F_c^\lambda(\theta,\psi)$ in \eqref{eq:def_Fcl} is differentiable at $\theta_0$ and
\begin{equation}
\nabla_\theta F_c^\lambda(\theta_0,\psi) = \E_{Z\sim\zeta}\left[ \left(\partial_\theta g(\theta_0,Z)\right)^T \nabla \psi^{c,\lambda}(g(\theta_0,Z))\right].
\end{equation}
Moreover, if $g$ is also $\mathscr{C}^1$, then $F_c^\lambda(\cdot,\psi)$ is $\mathscr{C}^1$ on $\Theta$.
\end{theorem}

\begin{proof} Let $\psi\in L^\infty$. The cost $c$ being $\mathscr{C}^1$ and $\calY$ a compact, the dominated convergence theorem gives that the $c,\lambda$-transform $\psi^{c,\lambda}$ \eqref{eq:def_psicl} is $\mathscr{C}^1$ and
\begin{equation}
\nabla \psi^{c,\lambda}(x) = \frac{\int_\calY \exp\left(\frac{\psi(y) - c(x,y)}{\lambda}\right) \nabla_x c(x,y)d\nu(y)}{\int_\calY \exp\left(\frac{\psi(y) - c(x,y)}{\lambda}\right)d\nu(y)}.
\end{equation}

From Hypothesis~\ref{hyp:lip}, $g$ is differentiable at a.e. $(\theta, z)$. This implies that for almost every $\theta$, $g$ is differentiable at $(\theta, z)$ for a.e. $z$ (and therefore admits a partial differential w.r.t. $\theta$). We set such a $\theta_0$. 
Let $p(\theta,z) := \psi^{c,\lambda}(g(\theta,z))$ so that 
\begin{equation}
F_c^\lambda(\theta,\psi) = \E_{Z\sim\zeta}[p(\theta,Z)] + \int \psi d\nu.
\end{equation}
For any $\psi$, since $\psi^{c,\lambda}$ is differentiable everywhere, there exists a neighborhood $\Omega_1$ of $\theta_0$ such that, for any $\theta\in\Omega_1$, $p(\theta,z)$ admits a partial differential at $\theta$ for almost every $z$. This partial differential writes
\begin{equation}\label{eq:diffpsic}
q(\theta,z)\hspace{-0.035cm}=\hspace{-0.03cm}\partial_\theta (\psi^{c,\lambda}\hspace{-0.02cm}(g(\theta,z)))\hspace{-0.035cm}=\hspace{-0.035cm} \left(\partial_\theta\hspace{-0.02cm} g(\theta,z)\right)^{\hspace{-0.02cm}T}\hspace{-0.04cm} \nabla \psi^{c,\lambda}(g(\theta,z)).
\end{equation}
Since $g$ satisfies Hypothesis~\ref{hyp:lip}, there also exists a neighborhood $\Omega_2$  of $\theta_0$ such that, for all $\theta\in\Omega_2$ and any $z\in\calZ$,
$\|\partial_\theta g(\theta,z) \| \leq L(\theta_0,z).
$ 
Since $\calX$ is compact, we have ${C := \sup_{\calX} \|\nabla \psi^{c,\lambda}\| < \infty}$.
Thus we get $\|q(\theta,z)\| \leq C L(\theta_0,z)$ with ${\E[L(\theta_0,Z)] < \infty}$.
Besides, for $\theta \in \Omega_1\cap\Omega_2$ and any $z \in \calZ$,
\begin{align}
&\frac{\left\|p(\theta, z) - p(\theta_0, z) - \langle q(\theta_0,z) , \theta - \theta_0 \rangle\right\|}{\|\theta - \theta_0\|}\label{eq:expec}\\
\leq& \sup_{\theta\in \Omega_1\cap\Omega_2} \|\partial_\theta g(\theta,z) \|C  + \| q(\theta_0,z) \| \\
\leq& 2 L(\theta_0,z)C < \infty\label{eq:domi}
\end{align}

We can thus apply the dominated convergence theorem that yields that $\theta \mapsto F_c^\lambda(\theta,\psi)$ is differentiable at $\theta_0$ with $\nabla_\theta F_c^\lambda(\theta_0,\psi) = \E[q(\theta_0,Z)]$, which is the desired formula.

When $g$ is $\mathscr{C}^1$, this is true for any $\theta_0 \in\Theta$ and since $\psi^{c,\lambda}$ is also $\mathscr{C}^1$, we get that $F_c^{\lambda}(\cdot,\psi)$ is also $\mathscr{C}^1$.
\end{proof}

\paragraph{Discussion.} Thanks to the entropic regularization, Theorem~\ref{thm:gradF} is valid for almost every $\theta$ and for \textbf{any} $\psi$.
This is an essential point for latter purpose, as it ensures  that the result stays true for a couple $(\theta,\psi^*)$ where $\psi^*$ is a Kantorovitch potential for $\theta$.
In the un-regularized case $\lambda=0$, the proof of Theorem \ref{thm:gradF} does not hold and additional assumptions on $g$ are required. Contrary to the $c,\lambda$-transform $\psi^{c,\lambda}$,  the $c$-transform \eqref{eq:ctransf} is indeed only differentiable almost everywhere. In this case, we only have that for any~$\psi$, $\theta \mapsto F_c^0(\theta,\psi)$  is differentiable for almost every $\theta$, with the `almost every $\theta$' depending on $\psi$. Notice that this is the actual result proved in \cite{arjovsky2017wgan}. 

Another point to discuss is the $\mathscr{C}^1$ assumption on the cost $c$ in Theorem \ref{thm:gradF}. This assumption does not cover the cost $c(x,y) = \|x-y\|$ from the original WGAN framework. This issue can be treated by taking care of the points where $c(x,y)$ is not differentiable. To do so, we can assume that for almost every $\theta$, the generator  $g_\theta\sharp\zeta$ does not put mass on the atoms of the measure $\nu$. This assumption is for instance  true if $g_\theta\sharp\zeta$ is absolutely continuous w.r.t. the Lebesgue measure. Hence, in order to keep the our statements as simple as possible, we still rely on the cost assumption $c\in\mathscr{C}^1(\calX\times\calY)$ in the following.

\subsection{Differentiation of $W_c^\lambda$}\label{sec:diffOT}
We can now state our main result ensuring the differentiability of $W_c^\lambda(\theta)$ with respect to the generator parameters~$\theta$.
We first demonstrate this claim for $\mathscr{C}^1$ generators $g$.

\begin{theorem}\label{thm:gc1}Let $c$ be a  $ \mathscr{C}^1(\calX\times\calY)$ cost,  for  $\calX$ and $\calY$ compact. If the generator  $g$ satisfies Hypothesis~\ref{hyp:lip} with $g \in \mathscr{C}^1(\Theta \times \calZ, \calX)$  then the mapping
\begin{equation}
W_c^\lambda :\theta \mapsto \OTc^\lambda(g_\theta\sharp\zeta,\nu)
\end{equation}
is $\mathscr{C}^1$ and we have, for any $\theta \in \Theta$,
\begin{align}
\nabla_\theta W_c^\lambda(\theta) &= \nabla_\theta F_c^\lambda(\theta, \psi_*)\nonumber\\
&= \mathbf{E}_{Z\sim\zeta}\hspace{-0.02cm}\left[\left(\partial_\theta g(\theta,z)\right)^T\nabla\psi_*^{c,\lambda}(g(\theta,z))\hspace{-0.02cm}\right]\hspace{-0.03cm},\label{eq:gradWreg}
\end{align}
where $\psi_* \in \argmax_\psi F_c^\lambda(\psi,\theta)$, and
\begin{equation}
\nabla\psi^{c,\lambda}(x) = \dfrac{\E_{Y\sim\nu}\left[\exp\left(\frac{\psi(Y) - c(x,Y)}{\lambda}\right)\nabla_x c(x,Y)\right]}{\E_{Y\sim\nu}\left[ \exp\left(\frac{\psi(Y) - c(x,Y)}{\lambda}\right)\right]}.
\end{equation}
\end{theorem}

The demonstration of this theorem is based on the application of the following result.

\begin{proposition}[A.1 from \cite{oyama2018envelope}]\label{thm:envoyama} Let $w(\theta) = \max_\psi h(\psi,\theta)$ and assume that
\begin{enumerate}[align=left,leftmargin=*]
\item there exists $ \psi^* :\theta \mapsto \psi^*(\theta)$ defined on a neighborhood of $\theta_0$ s.t. $h(\psi^*(\theta),\theta) = w(\theta)$ and $\psi^*$ continuous at $\theta_0$
\item $h$ is differentiable in $\theta$ at $(\psi^*(\theta_0),\theta_0)$ and $\nabla_\theta h$ continuous in $(\psi,\theta)$ at $(\psi^*(\theta_0),\theta_0)$
\end{enumerate}
Then $w$ is differentiable at $\theta_0$ and 
\begin{equation}
\nabla w(\theta_0) = \nabla_\theta h(\psi^*(\theta_0),\theta_0).
\end{equation}
\end{proposition}

We also need the following technical lemma.

\begin{lemma}\label{lem:selection}
There exists a continuous selection of Kantorovitch potentials, that is, a function ${\psi_* : \Theta \to \mathscr{C}(\calY)}$ which is continuous (for the uniform norm on $\mathscr{C}(\calY)$) and such that, for any $\theta$, $\psi_*(\theta) \in \argmax_\psi F_c^\lambda(\theta, \psi)$.
\end{lemma}
\begin{proof}
Let $\theta_0 \in\Theta$. 
First notice that the Kantorovitch potential $\psi$ solving \eqref{eq:def_Wl} is unique (up to a constant) since $\lambda >0$ and $c$ is bounded~\cite{genevay2019thesis}. 
We set an arbitrary $y_0 \in \calY$.
For all $\theta \in\Theta$, let us consider $\psi_\theta$ to be the Kantorovitch potential such that $\psi_\theta(y_0) = 0$. 
Since the cost $c$ is continuous on $\calX\times\calY$ compact, it is absolutely continuous. Lemma~\ref{lem:modulus} also provides that the cost has a bounded modulus of continuity $\omega$ that is shared with any Kantorovich potential $\psi_\theta$. This implies that the set $\{\psi_\theta\}_{\theta\in\Omega}$ is equicontinuous and that for all $\theta$, $\psi_\theta$ is bounded (independently of $\theta$) by $\sup_{y\in\calY}\omega(|y-y_0|)$. The Arzela-Ascoli theorem therefore implies that $\{\psi_\theta\}_{\theta\in\Omega}$ is relatively compact.

Now, by contradiction, assume that $\theta \mapsto \psi_\theta$ is not continuous at a point $\theta_0$. Then there exists $\varepsilon > 0$ and a sequence $(\theta_n)\in\Omega^\mathbf{N}$, such that $\theta_n \to \theta_0$ and
\begin{equation}
\|\psi_{\theta_n} - \psi_{\theta_0}\|_\infty > \varepsilon.
\end{equation}
Since $\{\psi_\theta\}_{\theta\in\Omega}$ is relatively compact, we can extract a subsequence $\theta_{r(n)}$ such that $\psi_{\theta_{r(n)}}$ converges uniformly towards a function $f$ in $\mathscr{C}(\calY)$. 
It follows that $\psi_{\theta_{r(n)}}^{c,\lambda}$ also converges towards $f^{c,\lambda}$. Let us denote $\mu_n =g_{\theta_n}\sharp\zeta$. This measure $\mu_n$ weakly converges towards $\mu_0$ and we can therefore write
\begin{align}
\OTc^\lambda(\mu_n,\nu) &= 
\int_\calX \psi_{\theta_{r(n)}}^{c,\lambda}(x)d \mu_n(x) + \int_\calY \psi_{\theta_{r(n)}}(y)d\nu(y)\\
&\underset{n\to\infty}{\longrightarrow} \int_\calX f^{c,\lambda}(x)d \mu_0(x) + \int_\calY f(y)d\nu(y).
\end{align}
Since $\OTc^\lambda(\mu_n,\nu) \to \OTc^\lambda(\mu_0,\nu)$, we get that $f$ is a Kantorovitch potential for $\OTc^\lambda(\mu_0,\nu)$. With $f(y_0) = \lim\psi_{\theta_{r(n)}}(y_0) = 0$ and the uniqueness up to a constant of Kantorovitch potentials we get $f =\psi_{\theta_0}$ which gives the contradiction and concludes the proof.
\end{proof}

\begin{proof}[Proof of Theorem~\ref{thm:gc1}]
The demonstration follows from the application of Proposition~\ref{thm:envoyama}. 
Lemma \ref{lem:selection} gives the first hypothesis for applying Proposition~\ref{thm:envoyama}. Theorem \ref{thm:gradF} for $g \in\mathscr{C}^1(\Theta\times\calZ)$ gives the second one and leads to the expression of $\nabla_\theta F_c^\lambda(\theta,\psi)$. 
\end{proof}

\paragraph{Case of $g$ not necessarily $\mathscr{C}^1$.}
When $g$  only satisfies Hypothesis~\ref{hyp:lip},  the gradient of $F_c^\lambda$ is not necessarily $\mathscr{C}^1$ as  seen in Section~\ref{sec:diffF}. Therefore $F_c^\lambda$ may not satisfy the second hypothesis of Proposition~\ref{thm:envoyama}, which is required to show the existence of the gradient in Theorem \ref{thm:gc1}. However we can still give a weaker result in this case, following the same sketch of proof as in \cite{arjovsky2017wgan}. We have already stated in Theorem~\ref{thm:gradF} that for any $\psi$, the gradient of $F_c^\lambda(\psi,\cdot)$ exists for almost every $\theta$. Therefore, we only need to show the existence of the gradient of $W_c^\lambda$ almost everywhere to ensure that Theorem~\ref{thm:env} holds for almost every $\theta$.

Let then demonstrate that $W_c^\lambda$ is differentiable for almost every $\theta$. Recall that $g$ satisfies Hypothesis~\ref{hyp:lip}. From relation  \eqref{eq:varW} in Theorem~\ref{thm:varW}, we get  that for all $\theta_1$, there exists a neighborhood $\Omega$ of $\theta_1$ such that for all $\theta_2\in\Omega$
\begin{align}
|W_c^\lambda(\theta_1) - W_c^\lambda(\theta_2)| &\leq \kappa\E_{Z\sim\zeta}\left[L(\theta_1, Z)\|\theta_1 - \theta_2\|\right]\nonumber\\
&\leq \kappa L(\theta_1)\|\theta_1-\theta_2\|.
\end{align}
The function $W_c^\lambda$ is thus locally Lipschitz and differentiable for almost every $\theta$ by Rademacher theorem.

\subsection{Back to the synthetic example}\label{sec:synth_solved}

We now study the synthetic example from Section~\ref{sec:failex} within the regularized framework. The regularized optimal transport between $\mu_\theta = \delta_\theta$ and $\nu = \frac{1}{2}\delta_{y_1} + \frac{1}{2}\delta_{y_2}$ writes in this case $W_c^\lambda(\theta) = \max_{\psi\in\mathbf{R}^2} F_c^\lambda(\psi,\theta)$ with $F_c^\lambda(\psi,\theta) = \frac{\psi_1+\psi_2}{2}-\lambda\log\left(\frac{1}{2}\left(\exp\left(\frac{\psi_1 - c(\theta,y_1)}{\lambda}\right) +\exp\left(\frac{\psi_2 - c(\theta,y_2)}{\lambda}\right)\right)\right) $.
The optimal $\psi^*$ maximizing $F_c^\lambda(\psi,\theta)$ satisfies $\psi^*_1 - \psi^*_2 = c(\theta,y_1) - c(\theta,y_2)$ and as in the unregularized case \eqref{eq:sol_unreg} we obtain
$$ \OTc^\lambda(\mu_\theta,\nu) = \frac{c(\theta,y_1)+ c(\theta,y_1)}{2}.$$
Hence both regularized and un-regularized problem share the same solution. In the un-regularized case, Proposition~\ref{prop:counterexample} states that the gradient of $W_c$ cannot be related to the gradient of $F_c$. On the other hand, Theorem~\ref{thm:gc1} stands for the regularized setting $W_c^\lambda$. 

Let us now highlight the numerical influence of the regularization. Starting from $\theta^0$ and for a given time step $\tau>0$, we consider the iterative algorithm
\begin{equation}\label{algo:synt}
\left\{\begin{array}{l}
\psi^k\in \argmax_\psi F_c^\lambda(\psi,\theta^k)\\
\theta^{k+1} = \theta^{k} -\tau \nabla_\theta F_c^\lambda(\psi^k,\theta^k).
\end{array}\right.
\end{equation}

Recall that the gradient of $\nabla_\theta F_c^\lambda(\psi^k,\theta^k)$ does not exist for $\lambda = 0$. In this case, we  propose to approximate $\psi^k$ with the gradient ascent procedure  proposed in~\cite{genevay2016ot} and then obtain the gradient via back-propagation. This corresponds to the alternate procedure in the WGAN framework of \cite{arjovsky2017wgan}.

We illustrate the behavior of Algorithm \eqref{algo:synt}, with $\tau=0.1$, for both un-regularized and regularized settings in Figure~\ref{fig:synthex}, where we take $y_1=(0,0)$ and $y_2=(0,1)$ and the $\mathscr{C}^1$ cost $c(x,y)=||x-y||^2$ corresponding to the $2$-Wasserstein distance. In such a setting, the optimal generator $g_\theta(z)=z-\theta$ is obtained for $\theta=\frac{y_1+y_2}{2}=(0,0.5)$. 
In the un-regularized case $\lambda = 0$, the gradient $ \nabla_\theta F_c^\lambda(\psi^k,\theta^k)$ alternates between  directions $(\theta-y_1)$ and $(\theta-y_2)$. The parameter $\theta^k$ thus oscillates around $(0,0.5)$. 
On the other hand, for the regularization parameter $\lambda=0.1$, the gradient  is well defined and $\theta^k$ converges monotonously towards $(0,0.5)$.


\section{Learning a generative model with regularized optimal transport}\label{sec:practicaluse}

We finally come to practical considerations.
The  gradient  formula \eqref{eq:gradWreg} given in Theorem~\ref{thm:gc1} takes the form of an expectation. In order to minimize the regularized optimal transport cost with respect to the generator parameter $\theta$, we perform a stochastic gradient optimization with  $(\partial_\theta g(\theta,z))^T\nabla\psi_*^{c,\lambda}(g(\theta,z))$, the term inside the expectation.
Such framework involves two limitations: (i) an optimal Kantorovitch potential $\psi_*^{c,\lambda}$ has to be approximated at each iteration; (ii) the stochastic gradient formula requires the computation of an expectation on the {\em whole} data distribution $\nu$.
The authors of~\cite{genevay2016ot} showed that in the case of a discrete target measure $\nu$, the first issue can be addressed with a stochastic gradient ascent on $\psi$. Moreover, if $\nu$ is discrete, the second point amounts to compute a mean over the dataset. This step is thus feasible, as the associated computational cost is linear with the number of data. 

When facing concrete applications, the target dataset is actually discrete. We thus propose to formulate the optimal transport in a semi-discrete way in Section~\ref{sec:semi-discrete}. We then present a numerical procedure in the spirit of the WGAN approach and some numerical examples in Section~\ref{sec:experiments}.

\subsection{Semi-discrete formulation}\label{sec:semi-discrete}

We consider a finite dataset $\{y_1,\ldots,y_n\}$, associated to the discrete target measure $\nu = \frac{1}{n}\sum_i \delta_{y_i}$. 
In this setting,  the {\em semi-discrete} formulation of the regularized cost $W_c^\lambda$ is
\begin{equation}\label{eq:semidiscrete}
W_c^\lambda(\theta) = \max_{\psi\in \mathbf{R}^n}\E_z\left[\psi^{c,\lambda}(g_\theta(z))\right] + \frac{1}{n}\sum_i\psi_i,
\end{equation}
with $\psi^{c,\lambda}(x) = -\lambda\log\left(\frac{1}{n}\sum_i \exp\left(\frac{\psi_i - c(x,y_i)}{\lambda}\right)\right)$. 
From Theorem~\ref{thm:gc1}, the gradient of $W_c^\lambda$  writes 
\begin{equation}\label{eq:gradsd}
\nabla_\theta W_c^\lambda(\theta) = \E_z\left[q(\theta, \psi^*,z)\right],
\end{equation}
with 
\begin{equation}\label{eq:q_grad_sto}
q(\theta, \psi^*,z) = \left(\partial_\theta g(\theta,z)\right)^T \sum_i \eta_i(g_\theta(z)) \nabla_\theta c(g_\theta(z),y_i)
\end{equation}
and \begin{equation}
\eta_i(x) = \frac{\exp\left(\frac{\psi^*_i -c(x,y_i)}{\lambda}\right)}{\sum_j \exp\left(\frac{\psi^*_j -c(x,y_j)}{\lambda}\right)}
\end{equation}
This formulation has four benefits: (i) the formulation \eqref{eq:semidiscrete} is known to be concave on $\psi$ and an optimum can be approximated with a stochastic gradient ascent procedure~\cite{genevay2016ot}; (ii) the dual potential $\psi$ does not need to be encoded with a neural network; (iii) the formulation holds for any cost $c$ that is $\mathscr{C}^1$; and (iv) the formula~\eqref{eq:gradsd} provides a stochastic gradient that can be computed numerically.


\subsection{Numerical illustrations}
\label{sec:experiments}

We propose in Algorithm~\ref{alg:algo} a numerical scheme based on a stochastic gradient descent of the regularized optimal transport cost. 
To perform the stochastic gradient descent on $\theta$, we use the Adam optimizer \cite{kingma2017adam} of the Pytorch library with default parameters, and a learning rate $lr=10^{-4}$. 
We train the network for $N=4000$ iterations.
The estimation of $\psi$ is done with $N_\psi=200$ iterations of the stochastic gradient ascent algorithm of~\cite{genevay2016ot} dedicated to entropy regularized optimal transport and we use batches of $K=100$ samples.
\begin{algorithm}[ht!]\small
    \caption{Learning Generative Model with stochastic gradient of semi-discrete entropic optimal transport}
   \label{alg:algo}
\begin{algorithmic}
   \STATE {\bfseries Inputs:} regularization parameter $\lambda$, cost function $c$, sample size $K$, number of iterations $N$ and $N_{\psi}$, training set $\{y_1, \ldots y_n\}$
   \STATE {\bfseries Output:} estimated generative model parameter $\theta^N$
   \STATE {\bfseries Initialisation:} $\psi^0=0$, random initialization of $\theta^0$
   \FOR{$k=1$ {\bfseries to} $N$}
   \STATE \textbullet~ Estimate $\psi^k$ with $N_{\psi}$ iterations of the stochastic gradient ascent method in~\cite{genevay2016ot} on a batch of size $K$
   \STATE \textbullet~ Draw a batch of $K$ samples $z$ from $\zeta$
   \STATE \textbullet~ Update $\theta^{k+1}$ with a stochastic gradient descent in the direction $q(\theta^k, \psi^k, z)$ given in \eqref{eq:q_grad_sto}
   \ENDFOR
\end{algorithmic}
\end{algorithm}

We now consider the application of Algorithm~\ref{alg:algo} to the learning of a generative model on the MNIST dataset with the cost $c(x,y) = \|x-y\|^2$. The generative model we considered in our experiments consists of four fully connected layers starting from a latent variable $z\in\calZ\subset\mathbf{R}^{128}$ to the image space $\calX\subset\mathbf{R}^{784}$, with intermediate dimensions $256$,  $512$ and $1024$. In this setting, the learning of the generator with a GPU Nvidia K40m takes approximately 2 hours. We show some generated digits in Figure~\ref{fig:MNIST} for different regularization parameters $\lambda$. 
\begin{figure}[t]
\newlength{\mylen}
\setlength{\mylen}{0.08\linewidth}
\newcommand{\sidecapY}[1]{{\begin{sideways}\parbox{2.1\mylen}{\centering #1}\end{sideways}}}
\newcommand{\sidecapYY}[1]{{\begin{sideways}\parbox{\mylen}{\centering #1}\end{sideways}}}
\centering
\setlength{\tabcolsep}{1pt}
\renewcommand{\arraystretch}{1.0}
\begin{tabular}{cccccccccc}
\multirow{ 2}{*}{\tiny\sidecapY{~\hspace*{3mm}$\lambda=0.001$}} 
&\includegraphics[width=\mylen]{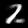} 
&\includegraphics[width=\mylen]{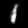}
&\includegraphics[width=\mylen]{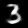} 
&\includegraphics[width=\mylen]{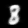}
&\includegraphics[width=\mylen]{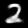}
&\includegraphics[width=\mylen]{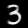}
&\includegraphics[width=\mylen]{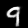}
&\includegraphics[width=\mylen]{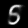}
&\includegraphics[width=\mylen]{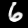}\vspace{-0.04cm}\\
&\includegraphics[width=\mylen]{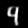} 
&\includegraphics[width=\mylen]{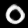}
&\includegraphics[width=\mylen]{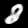} 
&\includegraphics[width=\mylen]{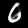}
&\includegraphics[width=\mylen]{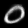}
&\includegraphics[width=\mylen]{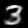}
&\includegraphics[width=\mylen]{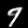}
&\includegraphics[width=\mylen]{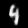}
&\includegraphics[width=\mylen]{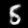}\vspace{0.04cm}
\\

\multirow{ 2}{*}{\tiny\sidecapY{\hfill$\lambda=0.01$}} 
&\includegraphics[width=\mylen]{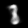} 
&\includegraphics[width=\mylen]{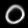} 
&\includegraphics[width=\mylen]{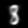}
&\includegraphics[width=\mylen]{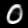} 
&\includegraphics[width=\mylen]{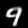} &\includegraphics[width=\mylen]{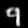}
&\includegraphics[width=\mylen]{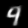}
&\includegraphics[width=\mylen]{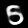}
&\includegraphics[width=\mylen]{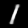}\vspace{-0.04cm}\\
&\includegraphics[width=\mylen]{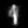} 
&\includegraphics[width=\mylen]{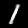} 
&\includegraphics[width=\mylen]{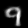}
&\includegraphics[width=\mylen]{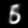} 
&\includegraphics[width=\mylen]{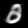} &\includegraphics[width=\mylen]{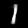}
&\includegraphics[width=\mylen]{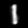}
&\includegraphics[width=\mylen]{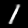}
&\includegraphics[width=\mylen]{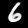}\vspace{0.04cm}
\\

{\tiny\sidecapYY{$\lambda$=\,0.1}} 
&\includegraphics[width=\mylen]{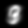} 
&\includegraphics[width=\mylen]{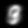} 
&\includegraphics[width=\mylen]{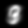}
&\includegraphics[width=\mylen]{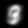} 
&\includegraphics[width=\mylen]{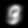} &\includegraphics[width=\mylen]{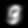}
&\includegraphics[width=\mylen]{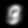}
&\includegraphics[width=\mylen]{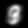}
&\includegraphics[width=\mylen]{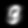}
\\
\end{tabular}
\caption{Random samples from generative models learned on the MNIST dataset with Alg.~\ref{alg:algo}, for $3$ regularization parameters $\lambda$. }
\label{fig:MNIST}
\end{figure}
This experiment shows that the proposed framework is able to learn a complex generative model, provided that the regularization parameter is sufficiently small. With high values of the regularization parameters such as $\lambda=10^{-1}$, the obtained generator realizes a compromise between all data points and concentrates to a mean image. 
This can be explained with the gradient expression~\eqref{eq:q_grad_sto} which, for $\lambda \to \infty$, pushes the generator towards a uniform average of the data points.
This also corroborates our observation that for high $\lambda$, the generator stabilizes in early iterations.

Last, observe that, as in WGAN approaches, we only use $N_\psi \ll n$ iterations for updating $\psi^k$. In practice, this number should depend on the dataset size $n$ in order to properly approximate the exact solution of the semi-discrete optimal transport problem. Another point is that $\psi$ depends on $n$ parameters. When $n$ is large, an interesting approximation could be to implicitly represent $\psi$ with a shallow neural network as in~\cite{seguy2018large}.

\section{Conclusion and discussion}
We have demonstrated that using optimal transport cost to train generative models, as popularized by the Wasserstein GAN framework \cite{arjovsky2017wgan}, raises theoretical issues when computing the gradient, even when assuming strong regularity properties on the generative model itself.
This flaw, illustrated on a toy example, can be circumvented by regularizing the optimal transport cost with entropy. 
The entropic regularization of optimal transport cost  \cite{cuturi2013sinkhorn} indeed enjoys interesting properties, such as an explicit dual formulation, fast computation and robustness to outliers~\cite{cuturi2014fast}.
As illustrated in experiments, and consistently with former works in the literature, the entropic smoothing may however yields an oversmoothed solution~\cite{cuturi2016smoothed}.
This can be cancelled by compensating the regularized cost bias, as demonstrated in \cite{feydy2019interpolating} and illustrated in~\cite{janati2020debiased}.
The resulting Sinkhorn divergence is nevertheless  challenging to compute when training generative models, as the generated distribution is continuous. 
Another interesting point is the extension of our analysis to other regularizations of optimal transport \cite{Dessein} 
or gradient penalty~\cite{gulrajani2017improved}. 

We took advantage in this work of the discrete nature of the target distribution, defined as a collection of training samples, to propose a simple optimization algorithm based on the stochastic gradient of the semi-discrete formulation of the regularized optimal transport.
As a corollary, the proposed framework is not restricted to the $1-$Wasserstein cost function anymore, as mostly done in the literature.
Recently, some works have been considering the training of generative models with other representation of the training set, for instance using differentiable data augmentation \cite{zhao2020differentiable}. 
The question of learning generative model with regularized optimal transport between continuous distributions, as recently studied in \cite{mensch2020online}, is an interesting perspective we leave for future work.

\bibliographystyle{plain}
\bibliography{refs}

\begin{thebibliography}{10}

\bibitem{amos2017input}
Brandon Amos, Lei Xu, and J~Zico Kolter.
\newblock Input convex neural networks.
\newblock In {\em International Conference on Machine Learning}, pages
  146--155. PMLR, 2017.

\bibitem{arjovsky2017wgan}
Martin Arjovsky, Soumith Chintala, and L{\'e}on Bottou.
\newblock Wasserstein generative adversarial networks.
\newblock In {\em International Conference on Machine Learning}, pages
  214--223, 2017.

\bibitem{Chen2019}
Yucheng Chen, Matus Telgarsky, Chao Zhang, Bolton Bailey, Daniel Hsu, and Jian
  Peng.
\newblock A gradual, semi-discrete approach to generative network training via
  explicit wasserstein minimization.
\newblock In {\em International Conference on Machine Learning}, pages
  1071--1080. PMLR, 2019.

\bibitem{cuturi2013sinkhorn}
Marco Cuturi.
\newblock Sinkhorn distances: Lightspeed computation of optimal transport.
\newblock In {\em Advances in neural information processing systems}, pages
  2292--2300, 2013.

\bibitem{cuturi2014fast}
Marco Cuturi and Arnaud Doucet.
\newblock Fast computation of wasserstein barycenters.
\newblock In {\em International conference on machine learning}, pages
  685--693. PMLR, 2014.

\bibitem{cuturi2016smoothed}
Marco Cuturi and Gabriel Peyr{\'e}.
\newblock A smoothed dual approach for variational wasserstein problems.
\newblock {\em SIAM Journal on Imaging Sciences}, 9(1):320--343, 2016.

\bibitem{Dessein}
Arnaud Dessein, Nicolas Papadakis, and Jean-Luc Rouas.
\newblock Regularized optimal transport and the {ROT} mover's distance.
\newblock {\em Journal of Machine Learning Research}, 19, 2018.

\bibitem{feydy2019interpolating}
Jean Feydy, Thibault S{\'e}journ{\'e}, Fran{\c{c}}ois-Xavier Vialard, Shun-ichi
  Amari, Alain Trouv{\'e}, and Gabriel Peyr{\'e}.
\newblock Interpolating between optimal transport and mmd using sinkhorn
  divergences.
\newblock In {\em The 22nd International Conference on Artificial Intelligence
  and Statistics}, pages 2681--2690. PMLR, 2019.

\bibitem{genevay2019thesis}
Aude Genevay.
\newblock {\em Entropy-regularized optimal transport for machine learning}.
\newblock PhD thesis, Paris Sciences et Lettres, 2019.

\bibitem{genevay2016ot}
Aude Genevay, Marco Cuturi, Gabriel Peyr{\'e}, and Francis Bach.
\newblock Stochastic optimization for large-scale optimal transport.
\newblock In {\em Advances in Neural Information Processing Systems}, pages
  3440--3448, 2016.

\bibitem{genevay2018learning}
Aude Genevay, Gabriel Peyr{\'e}, and Marco Cuturi.
\newblock Learning generative models with sinkhorn divergences.
\newblock In {\em International Conference on Artificial Intelligence and
  Statistics}, pages 1608--1617. PMLR, 2018.

\bibitem{goodfellow2014gan}
Ian Goodfellow, Jean Pouget-Abadie, Mehdi Mirza, Bing Xu, David Warde-Farley,
  Sherjil Ozair, Aaron Courville, and Yoshua Bengio.
\newblock Generative adversarial nets.
\newblock In {\em Advances in Neural Information Processing Systems}, pages
  2672--2680, 2014.

\bibitem{gulrajani2017improved}
Ishaan Gulrajani, Faruk Ahmed, Martin Arjovsky, Vincent Dumoulin, and Aaron~C
  Courville.
\newblock Improved training of wasserstein {GAN}s.
\newblock In {\em Advances in neural information processing systems}, pages
  5767--5777, 2017.

\bibitem{houdard2020wasserstein}
Antoine Houdard, Arthur Leclaire, Nicolas Papadakis, and Julien Rabin.
\newblock Wasserstein generative models for patch-based texture synthesis.
\newblock {\em arXiv preprint arXiv:2007.03408}, 2020.

\bibitem{janati2020debiased}
Hicham Janati, Marco Cuturi, and Alexandre Gramfort.
\newblock Debiased sinkhorn barycenters.
\newblock In {\em International Conference on Machine Learning}, pages
  4692--4701. PMLR, 2020.

\bibitem{kingma2017adam}
Diederik~P. Kingma and Jimmy Ba.
\newblock Adam: A method for stochastic optimization.
\newblock In {\em International Conference on Learning Representations}, 2015.

\bibitem{Kingma2014}
Diederik~P. Kingma and Max Welling.
\newblock {Auto-Encoding Variational Bayes}.
\newblock In {\em International Conference on Learning Representations}, 2014.

\bibitem{korotin2019wasserstein}
Alexander Korotin, Vage Egiazarian, Arip Asadulaev, Alexander Safin, and Evgeny
  Burnaev.
\newblock Wasserstein-2 generative networks.
\newblock {\em arXiv preprint arXiv:1909.13082}, 2019.

\bibitem{liu2018two}
Huidong Liu, GU~Xianfeng, and Dimitris Samaras.
\newblock A two-step computation of the exact {GAN} wasserstein distance.
\newblock In {\em International Conference on Machine Learning}, pages
  3159--3168. PMLR, 2018.

\bibitem{mensch2020online}
Arthur Mensch and Gabriel Peyr{\'e}.
\newblock Online sinkhorn: Optimal transport distances from sample streams.
\newblock {\em Advances in Neural Information Processing Systems}, 33, 2020.

\bibitem{miyato2018spectral}
Takeru Miyato, Toshiki Kataoka, Masanori Koyama, and Yuichi Yoshida.
\newblock Spectral normalization for generative adversarial networks.
\newblock {\em arXiv preprint arXiv:1802.05957}, 2018.

\bibitem{oyama2018envelope}
Daisuke Oyama and Tomoyuki Takenawa.
\newblock On the (non-) differentiability of the optimal value function when
  the optimal solution is unique.
\newblock {\em Journal of Mathematical Economics}, 76:21--32, 2018.

\bibitem{WGANLP}
Henning Petzka, Asja Fischer, and Denis Lukovnicov.
\newblock On the regularization of wasserstein {GAN}s.
\newblock In {\em International Conference on Learning Representations}, 2018.

\bibitem{sanjabi2018convergence}
Maziar Sanjabi, Jimmy Ba, Meisam Razaviyayn, and Jason~D Lee.
\newblock On the convergence and robustness of training {GAN}s with regularized
  optimal transport.
\newblock {\em arXiv preprint arXiv:1802.08249}, 2018.

\bibitem{santambrogio2015ot}
Filippo Santambrogio.
\newblock Optimal transport for applied mathematicians.
\newblock {\em Progress in Nonlinear Differential Equations and their
  applications}, 87, 2015.

\bibitem{seguy2018large}
Vivien Seguy, Bharath~Bhushan Damodaran, R{\'e}mi Flamary, Nicolas Courty,
  Antoine Rolet, and Mathieu Blondel.
\newblock Large-scale optimal transport and mapping estimation.
\newblock In {\em International Conference in Learning Representations}, 2018.

\bibitem{zhao2020differentiable}
Shengyu Zhao, Zhijian Liu, Ji~Lin, Jun-Yan Zhu, and Song Han.
\newblock Differentiable augmentation for data-efficient {GAN} training.
\newblock {\em arXiv preprint arXiv:2006.10738}, 2020.

\end{thebibliography}

\end{document}